\documentclass{article}

\usepackage[preprint]{neurips_2026}

\usepackage[utf8]{inputenc} 
\usepackage[T1]{fontenc}    
\usepackage{hyperref}       
\usepackage{url}            
\usepackage{booktabs}       
\usepackage{amsfonts}       
\usepackage{nicefrac}       
\usepackage{microtype}      
\usepackage{xcolor}         
\usepackage{subcaption}
\usepackage{graphicx}

\usepackage{amsmath}
\usepackage{amssymb}
\usepackage{mathtools}
\usepackage{amsthm}
\usepackage{bm}                 
\usepackage{geometry}
\usepackage{pgfplots}

\newtheorem{proposition}{Proposition}
\newtheorem{definition}{Definition}
\newtheorem{lemma}{Lemma}

\title{B-cos GNNs: Faithful Explanations through Dynamic Linearity}

\author{%
  Joschka Groß\textsuperscript{1,2,3,*}\\
  \And
  Mohammad Shaique Solanki\textsuperscript{1,*}
  \And
  Verena Wolf\textsuperscript{1,2}
}
\newcommand{\norm}[1]{\left\lVert#1\right\rVert}
\newcommand{\abs}[1]{\left\lvert#1\right\rvert}

\newcommand{\ie}{i.\,e.\ }
\newcommand{\eg}{e.\,g.\ }

\newcommand{\dnode}{\ensuremath{p}}

\begin{document}

\maketitle
\vspace{-2em}
\begin{center}
\textbf{\textsuperscript{1}} Saarland Informatics Campus, Saarland University\\
\textbf{\textsuperscript{2}} DFKI~~
\textsuperscript{3} \texttt{jgross@cs.uni-saarland.de}\\
\textsuperscript{*} Equal Contribution.
\end{center}
\vspace{1em}


\begin{abstract}
We introduce B-cos GNNs, an inherently explainable class of graph neural networks whose predictions decompose exactly into per-node, per-feature contributions via a single input-dependent linear map.
B-cos GNNs use linear (sum-based) aggregation and replace non-linear message and update functions with B-cos transforms.
This induces meaningful, task-specific weight-input alignment that is directly accessible through the model's dynamic linearity.
Instance-level explanations follow from a single forward and backward pass, requiring no auxiliary explainer, modified learning objective, or perturbation procedure. 
Instantiated as a GIN, our approach trades small losses in predictive accuracy for state-of-the-art explainability across diverse synthetic and real-world benchmarks, producing explanations orders of magnitude faster than post-hoc baselines.
\end{abstract}

\section{Introduction}
\label{sec: intro}

Graph Neural Networks (GNNs) have emerged as powerful tools for learning from relational data, with state-of-the-art results across recommendation systems~\cite{ying2018graph}, antibiotic discovery~\cite{stokes2020deep}, weather forecasting~\cite{lam2023learning}, and particle physics~\cite{shlomi2020graph}. While predictive accuracy is necessary, it is rarely sufficient: high-stakes deployment, scientific discovery, and the detection of shortcut learning all require understanding \emph{why} a model makes certain predictions. These needs are best met by \emph{faithful}~\cite{jacovi2020towards} explanations, i.e., explanations that reflect the computation the model actually performed rather than approximating it through a separately optimized surrogate. Without faithfulness, plausible-looking explanations may obscure spurious correlations or biased reasoning, undermining both trust and scientific validity.

Existing methods for instance-level GNN explainability fall into three broad categories. \textbf{Generic post-hoc attribution} methods include gradient-based approaches such as Saliency Maps~\cite{simonyan2013deep}, Integrated Gradients~\cite{sundararajan2017axiomatic}, and GradCAM~\cite{pope2019explainability}, as well as decomposition methods like LRP~\cite{bach2015pixel,baldassarre2019explainability}. \textbf{Graph-specific post-hoc} methods such as GNNExplainer~\cite{ying2019gnnexplainer}, PGExplainer~\cite{luo2020parameterized}, GraphMask~\cite{schlichtkrull2020graphmask}, and SubgraphX~\cite{yuan2021on} treat the trained GNN as a black box and optimize an auxiliary model to extract small explanatory subgraphs. While effective, post-hoc methods typically require careful regularization tuning and be fragile~\cite{li2024graph}. Motivated by these limitations, a growing line of work argues that high-stakes domains should rely on models that are interpretable by construction~\cite{rudin2019stop}. The third category, \textbf{inherently explainable GNNs} such as IB-subgraph~\cite{yu2020gib}, DIR~\cite{wu2022dir}, and GSAT~\cite{miao2022interpretable}, integrates the explanation mechanism into the model itself, but typically at the cost of  architectural complexity: auxiliary subgraph generators, rationale modules, or modified information-bottleneck objectives.

\begin{figure}[ht]
    \centering
    \begin{subfigure}{0.48\textwidth} 
        \includegraphics[width=0.98\textwidth]{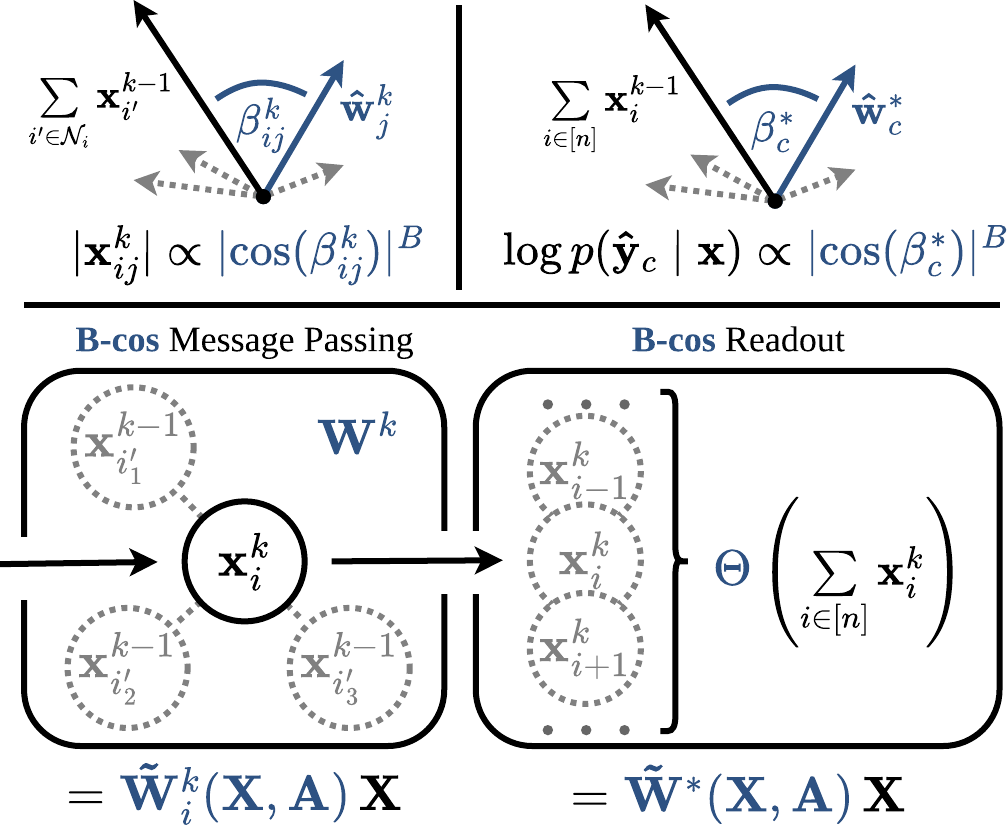} %
        \subcaption{\textbf{B-cos GNNs} use B-cos transforms~\cite{bohle2022b} in place of linear projections in message passing and readout.
        Their node-embeddings and predictions can be summarized by a single input-dependent linear transform.
        They thus facilitate easily accessible and highly interpretable weight-input alignment.}
        \label{subfig:method_figure_1}
    \end{subfigure}
    \hfill
    \begin{subfigure}{0.44\textwidth} 
        \includegraphics[width=0.98\textwidth]{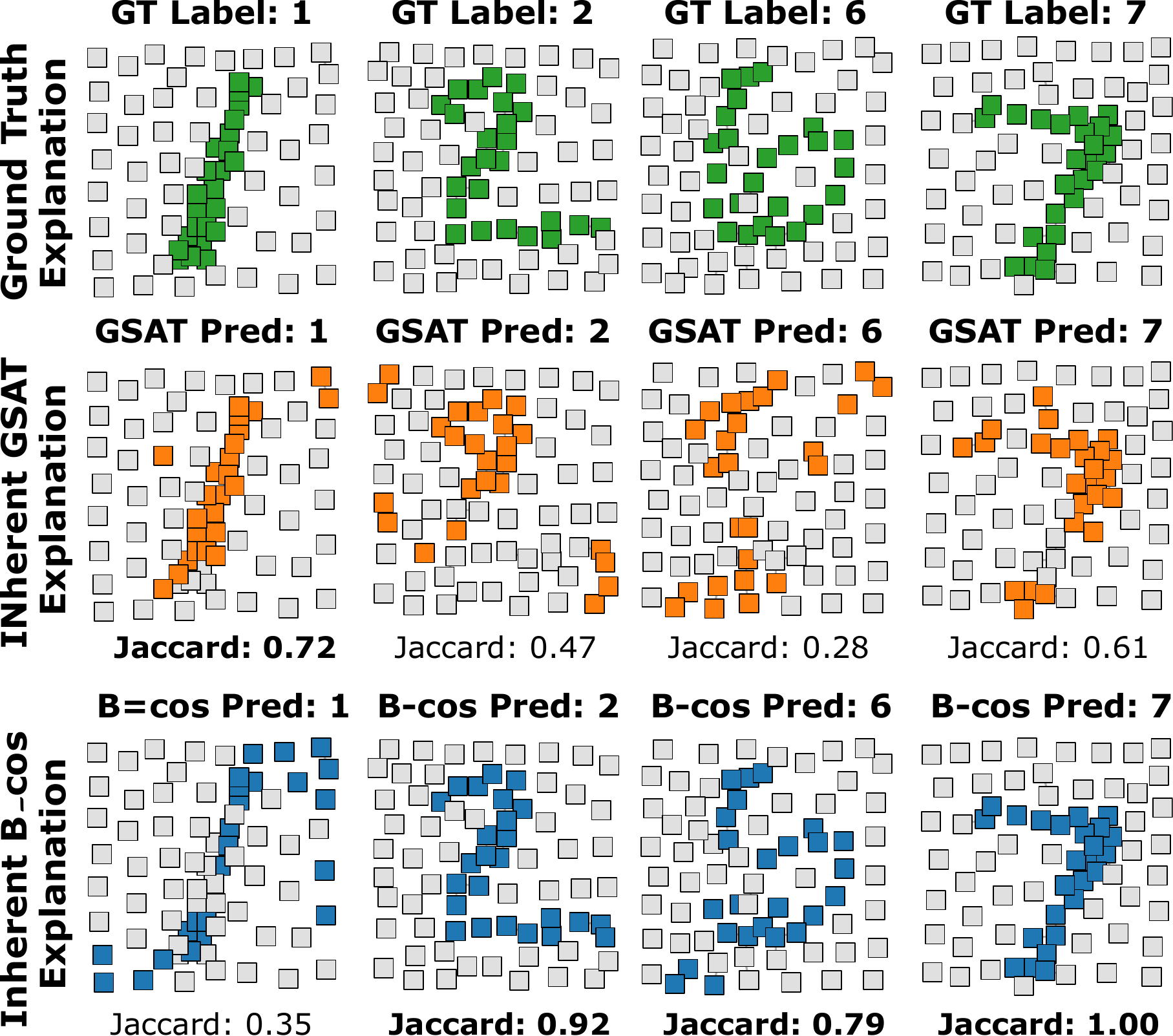} %
        \subcaption{Ground truth, GSAT~\cite{miao2022interpretable} and B-cos GNN \textbf{explanations on MNIST-75sp} (superpixel nodes)~\cite{knyazev2019understanding} examples. Corner contributions originate from the positional $(x,y)$ feature channels. 
        Additional per-class random samples are provided in App.~\ref{app:explanation_examples}).}
        \label{subfig:example_explanations}
    \end{subfigure}
    \caption{Methodological overview and example explanations for our method and the inherently interpretable state-of-the-art baseline GSAT.}
\end{figure}

In this work, we introduce \emph{B-cos GNNs}, a class of inherently explainable GNNs that achieve state-of-the-art explanation accuracy with none of this complexity. Our models leverage B-cos transforms~\cite{bohle2022b} as  replacements for the standard pairing of affine transformations and non-linear activations. B-cos transforms satisfy \emph{dynamic linearity}: their output can be exactly summarized by an input-dependent linear map whose entries quantify weight--input alignment (Sec.~\ref{sec:preliminaries}). We show that this property carries over to GNNs under sum-based aggregation, so that predictions of a B-cos GNN decompose exactly into per-node, per-feature contributions (Fig.~\ref{subfig:method_figure_1}) obtainable from a single forward and backward pass. The sole architectural change is replacing standard fully connected subnetworks with B-cos networks; no auxiliary explainer, no modified training objective, and no perturbation procedure are required. The construction introduces a single hyperparameter $B$ that trades off interpretability against predictive flexibility.

We instantiate our recipe in the Graph Isomorphism Network (GIN) framework~\cite{xu2018powerful}, chosen for its sum-based aggregation and modularity, and extend it to edge-attributed graphs (GINE). Our contributions are:
\begin{itemize}
    \item We identify the design principles under which a GNN inherits dynamic linearity, yielding predictions that decompose exactly into per-node, per-feature contributions in a single backward pass.
    \item We instantiate these principles in the GIN architecture and extend them to edge-attributed variants (GINE).
    \item On synthetic and chemical benchmarks with ground-truth motifs, B-cos GINs achieve state-of-the-art alignment with the true rationale, outperforming both post-hoc explainers (GNNExplainer, IG) and the inherent baseline GSAT, while generating explanations orders of magnitude faster than optimization-based post-hoc methods.
\end{itemize}

\section{Preliminaries}
\label{sec:preliminaries}

\subsection{Graph neural networks}

GNNs operate over attributed graphs, which we define as tuples $\mathbf G = (\mathbf X, \mathbf A)$, where $\mathbf X \in \mathbb R^{n \times \dnode}$ assigns a $\dnode$-dimensional feature vector $\mathbf x_i^{(0)}$ to each node ($i=1,\ldots,n$). For clarity, we develop our analysis for graphs without edge features; the extension to edge-attributed graphs is discussed at the end of \autoref{sec:bcos_gnn}.
The adjacency matrix $\mathbf A \in \mathbb R^{n \times n}$ encodes the graph topology; for unweighted graphs, $\mathbf A_{ij} \in \{0,1\}$ indicates the presence of an edge between nodes $i$ and $j$.

We consider standard GNNs consisting of multiple message-passing layers~\cite{velivckovic2023everything}, each transforming node embeddings in three steps: messages are constructed from neighboring nodes, aggregated according to the topology, and used to update the central node's embedding. Formally, the output of the $k$-th layer at node $i$ is
\begin{align}
    \label{eq:message_passing}
    \mathbf x_i^{(k)} = \Psi^{(k)}\!\left(\mathbf x_i^{(k-1)}, \bigoplus_{j} \mathbf A_{ij}\,\Phi^{(k)}\!\left(\mathbf x_j^{(k-1)}\right)\right),
\end{align}
where $\Phi^{(k)}$ and $\Psi^{(k)}$ are functions (typically neural networks) that construct messages and update node embeddings respectively, and $\bigoplus$ is a permutation-invariant aggregation operator (\eg sum, mean, max). Our analysis in \autoref{sec:bcos_gnn} requires $\bigoplus$ to be summation.

After the final ($L$-th) layer, a readout module $\Theta$ produces the prediction $\mathbf{\hat y}$, either node-wise for node-level tasks ($\mathbf{\hat y} = \Theta(\mathbf x^{(L)}) \in \mathbb R^{n \times c}$) or on a globally pooled representation for graph-level tasks ($\mathbf{\hat y} = \Theta(\sum_{i=1}^n \mathbf x_i^{(L)}) \in \mathbb R^{c}$), where $c$ is task-dependent (\eg $c=1$ for binary classification or regression).

\subsection{B-cos networks}
\label{sec:bcos_trans}
The B-cos transformation~\cite{bohle2022b} is a modified, bias-free linear transformation that promotes weight-input alignment, controlled by a hyperparameter $B \in [1, \infty)$. B-cos networks solve tasks by aligning weights with inputs, and their output can be rewritten as a single input-dependent linear transform whose entries quantify exactly this alignment for the given input. Such networks are competitive in both computer vision~\cite{bohle2022b, bohle2024b} and natural language~\cite{bohle2024b, wang2025b} while exhibiting strong interpretability.

\paragraph{Contrasting linear and B-cos transforms}
For the sake of direct comparison, we state the $j$-th activation of a (zero bias) linear transformation $f_{\mathbf W}$ and B-cos transformation $f_{(\mathbf W, B)}$ side-by-side.
The former is parametrized only by a weight matrix $\mathbf W \in \mathbb R^{q \times p}$ and the latter introduces an additional hyperparameter $B \in [1, \infty)$.
Let $\mathbf x \in \mathbb R^{p}$ be an input vector and let $\mathbf w_j$ denote the $j$-th row of the weight matrix. Then, their $j$-th activations can be expressed as
\begin{align}
    \label{eq:linear}
    f_{\mathbf W}(\mathbf x)_j &= \norm{\mathbf w_j} \norm{\mathbf x} c(\mathbf x, \mathbf w_j), \,\text{and } \\
    \label{eq:bcos}
    f_{(\mathbf W, B)}(\mathbf x)_j &= \norm{\mathbf{\hat{w}}_j} \norm{\mathbf x} \abs{c(\mathbf x, \mathbf{\hat{w}}_j)}^B  \text{sgn}(c(\mathbf x, \mathbf{\hat{w}}_j))
\end{align}
respectively.
The hat-operator scales weight vectors such that $\mathbf{\hat w}_j$ has unit norm and $c(\mathbf x, \mathbf{\hat{w}}_j) \in [-1,1]$ denotes the cosine-similarity (weight-input alignment) between $\mathbf x$ and $\mathbf w_j$.

B-cos activations are upper-bounded by $\norm{\mathbf x}$ (assuming unit-norm weights). Increasing   $B$ leads to a sharper activation landscape where higher activations require higher weight-input alignment. For values of $B > 1$, the B-cos transformation is inherently non-linear, which means the inclusion of additional non-linear activation functions is not strictly necessary to model non-linear data.

\paragraph{Dynamic linearity}
Most importantly, B-cos transforms satisfy \emph{dynamic linearity}.
\begin{definition} 
\label{def:dynlin}
A function $f_\theta: \mathbb R^p \to \mathbb R^q$ with parameters $\theta$ satisfies \emph{dynamic linearity} if it can be rewritten as
$$
f_\theta(\mathbf x) = \mathbf{W}_\theta(\mathbf x) \mathbf x,
$$
such that the \emph{dynamic weights} $\mathbf{W}_\theta(\mathbf x) \in \mathbb R^{q \times p}$ only depend on the input $\mathbf x \in \mathbb R^p$ and parameters $\theta$.
\end{definition}
Let   $c(\mathbf x, \mathbf{\hat W})$ denote the vector of row-wise cosine similarities, i.e. $c(\mathbf x, \mathbf{\hat W})_j=c(\mathbf x, \mathbf{\hat w}_j)$.
For a B-cos transform $f_\theta$ with parameters $\theta = (\mathbf W, B)$,  \cite{bohle2022b} show that
\begin{align}
\label{eq:input_dependent_linear_transform}
f_\theta(\mathbf x) &= \mathbf{\tilde{W}}_\theta(\mathbf x) \mathbf x,
\end{align}
where the B-cos dynamic weights are given by
\begin{align}
\label{eq:input_dependent_weights}
\mathbf{\tilde W}_\theta(\mathbf x) &= \mathrm{diag}\!\big(\abs{c(\mathbf x, \mathbf{\hat W})}^{B - 1}\big)\, \mathbf{\hat W}.
\end{align}
Here, $\mathbf{\hat W}$ is the row-wise unit-norm rescaling of $\mathbf W$, and $\abs{\cdot}^{B-1}$ is applied element-wise to the cosine-similarity vector. Intuitively, $\mathbf{\tilde W}_\theta(\mathbf x)$ is $\mathbf{\hat W}$ with each row $j$ scaled by a factor that vanishes when $\mathbf x$ is poorly aligned with $\hat{\mathbf w}_j$ and approaches one when alignment is high.
Multiple B-cos transforms can be composed into deeper networks. Any composition
$$
g_{(\theta_1, \ldots, \theta_L)}(\mathbf x) = (f_{\theta_L} \circ \ldots \circ f_{\theta_1})(\mathbf x)
$$
of $L$ dynamically linear transforms $f_{\theta_1}, \ldots, f_{\theta_L}$ is itself dynamically linear, with 
weights
\begin{align}
\mathbf W_{(\theta_1, \ldots, \theta_L)}(\mathbf x) = \tilde{\mathbf W}_{\theta_L}(a_L) \cdots \tilde{\mathbf W}_{\theta_1}(a_1), \qquad a_k = (f_{\theta_{k-1}} \circ \cdots \circ f_{\theta_1})(\mathbf x),
\label{eq:finalweights}
\end{align}
where $a_k$ denotes the input to the $k$-th transform, i.e., the output of the first $k-1$ transforms applied to $\mathbf x$ (with $a_1 = \mathbf x$ by convention).
The \emph{contribution map} of the above model is given by
\begin{align}
\label{eq:contribution_map}
\mathbf M(\mathbf x) = \mathbf W_{(\theta_1, \ldots, \theta_L)}(\mathbf x) \odot \mathbf x,
\end{align}
where $\odot$ denotes element-wise multiplication of each row of $\mathbf W_{(\theta_1,\ldots,\theta_L)}(\mathbf x)$ with $\mathbf x$ (broadcasting along rows).
Crucially for interpretability, the output $g_{(\theta_1, \ldots, \theta_L)}(\mathbf x)$ equals the row-sum $\sum_i \mathbf M(\mathbf x)_i$, so that the entry $\mathbf M(\mathbf x)_{i,j}$ exactly quantifies the contribution of the $i$-th input feature to the $j$-th output element.

For B-cos networks, this decomposition carries genuine interpretive content beyond formal exactness. Because each B-cos transform reaches high activations only when its weights align with its inputs (\autoref{eq:bcos})the input-dependent matrix $\mathbf W_{(\theta_1,\ldots,\theta_L)}(\mathbf x)$ encodes the directions in input space that the network has learned to respond to — and the contributions $\mathbf M(\mathbf x)_{i,j}$ inherit this alignment structure. 
By contrast, other piecewise-linear models, such as ReLU networks, admit a similar dynamic-linear rewriting, but their weights at a given input do not, in general, carry an analogous alignment interpretation. We exploit this property when extending B-cos networks to graphs in the next section.

\section{B-cos graph neural networks}
\label{sec:bcos_gnn}

\paragraph{A recipe for inherently interpretable GNNs}
For compact notation, we denote the dimension of node embeddings at layer $k$ by $p_k$ (with $p_0 = \dnode$) and write the embeddings of all $n$ nodes as a flattened vector $\mathbf x^{(k)} \doteq \mathrm{vec}(\mathbf X^{(k)}) \in \mathbb R^{n p_k}$ throughout this section. We write $\mathbf x \doteq \mathbf x^{(0)}$ for the initial node features.

To construct B-cos GNNs whose predictions can be summarized through dynamic linearity (\autoref{def:dynlin}), we restrict the message and update functions $\Phi^{(k)}, \Psi^{(k)}$ in \autoref{eq:message_passing}, the readout $\Theta$, and any global pooling to compositions of B-cos transforms and constant linear operations combined via weighted summation. Each such component is dynamically linear individually, and weighted summation preserves this property (App.~\ref{app:dynlin_proofs}). Consequently, every node embedding and the resulting prediction of our B-cos GNN can be written as
\begin{align}
    \label{eq:layer_dynlin}
    \mathbf x_i^{(k)} = \mathbf W_i^{(k)}(\mathbf X, \mathbf A)\,\mathbf x,
    \qquad
    \mathbf{\hat y} = \mathbf W^*(\mathbf X, \mathbf A)\,\mathbf x,
\end{align}
where the input-dependent matrices $\mathbf W_i^{(k)}$ and $\mathbf W^*$ depend only on $\mathbf x$  and the GNN parameters.

\paragraph{Explanatory subgraphs from contribution maps}
For graph-level predictions, the contribution map (\autoref{eq:contribution_map}) of our B-cos GNN takes the form
\begin{align}
    \label{eq:gnn_contribution_map}
    \mathbf M^*(\mathbf X, \mathbf A) = \mathbf W^*(\mathbf X, \mathbf A) \odot \mathbf x \;\in\; \mathbb R^{c \times np_0},
\end{align}
where row $r$ of $\mathbf M^*$ quantifies the contribution of every initial node feature to the $r$-th output. To explain why the model assigned a particular class $\hat r$ to an input graph, we select the corresponding row of $\mathbf M^*$ and aggregate the contributions per node, $s_i = \sum_j \mathbf M^*_{\hat r, (i,j)}$, where the index $(i,j)$ ranges over the $j$-th feature of the $i$-th node. 
Following the standard formulation~\cite{yuan2022explainability}, we then represent the explanation by the node subset $S \subseteq \{1,\ldots,n\}$ obtained by ranking nodes according to $s_i$ and selecting the top $m$, or by selecting the smallest $S$ whose total contribution exceeds a target fraction of the model response. The induced subgraph is the explanation.

\paragraph{B-cos GINs}
We instantiate the construction from \autoref{sec:bcos_gnn} in the GIN framework~\cite{xu2018powerful}, whose sum-based aggregation matches the requirements above. The standard GIN update is
\begin{align}
\label{eq:bcos_gin_update}
    \mathbf x_i^{(k)} = \mathrm{MLP}^{(k)}\!\left(
        (1 + \epsilon^{(k)})\,\mathbf x_i^{(k-1)} + \sum_{j \in \mathcal N(i)} \mathbf x_j^{(k-1)}
    \right),
\end{align}
where $\epsilon^{(k)} \in \mathbb R$ is a (potentially learnable) scalar. To obtain a B-cos GIN, we replace each MLP with a stack of B-cos transforms, taking on the role of the update module $\Psi^{(k)}$. Since GIN constructs no explicit messages, $\Phi^{(k)}$ is the identity. The remaining components, the affine combination $(1+\epsilon^{(k)}) \mathbf x_i^{(k-1)} + \sum_j \mathbf x_j^{(k-1)}$ and the final summation, are constant linear operations and thus preserve dynamic linearity (Def. \ref{def:dynlin}). The computation of contribution maps leverages~\cite{bohle2022b}.

\paragraph{Extension to edge features}
We treat edge features as additive offsets to neighbor embeddings, replacing $\mathbf x_j^{(k-1)}$ in the GIN update by $\mathbf x_j^{(k-1)} + \mathbf e_{ji}$. This is a constant linear operation and preserves dynamic linearity. We refer to the resulting model as B-cos GINE (App.~\ref{app:gine_architectures}) and use it for our edge-attributed datasets.

\section{Experiments}
We evaluate B-cos GINs both in terms of predictive performance and the quality of their explanations and compare them to conventional GINs explained using GNNExplainer and IG as baseline post-hoc explainers as well as an inherently interpretable GSAT-GIN baseline.
Moreover, we also analyze the sensitivity of our B-cos GINs to the $B$ hyperparameter used to control the weight-input alignment strength. Throughout our empirical evaluation, we adhere to a consistent experimental protocol. 

All datasets are partitioned into training, validation, and test sets using a stratified 70\%/20\%/10\% split, with two exceptions: for MNIST-75sp we use 20{,}000 training, 5{,}000 validation, and 1{,}000 test graphs to reduce training and inference time given the larger graph sizes and Di-Halo-Benzene dataset where we used 80\%/10\%/10\%. 
We implement model training and evaluation logic using the PyTorch~\cite{paszke2019pytorch}, PyTorch geometric~\cite{paszke2019pytorch} and B-cos v2~\cite{BcosV2} packages.
To ensure statistical reliability, every experiment is executed across five independent random seeds, with all metrics reported as the mean alongside the standard deviation. 
Experiments were conducted on a single machine equipped with a AMD EPYC 7662 64-Core Processor and NVIDIA A100-PCIE-40GB GPU.
We do not release code but provide a reproducibility statement in \autoref{app:reproducibility}.

\subsection{Datasets}
\label{sec:datasets}
Recent work has shown that many standard graph-XAI benchmarks admit trivial solutions, weak predictors, or redundant rationales~\cite{faber2021comparing,agarwal2023evaluating,coupette2025no}, and that conclusions drawn from a small set of synthetic datasets often do not transfer to real-world graphs~\cite{amara2022graphframex,amara2023ginxeval}. 
We therefore select benchmarks along three axes: (i) coverage of four 
distinct graph domains (chemistry, vision-as-graph, stochastic block 
models, synthetic motifs), with orthogonal variation in task 
granularity, edge attributes, and rationale source (Tab.~\ref{tab:datasets_overview}); 
(ii) availability of verified ground-truth rationales for explainability 
evaluation; and (iii) compliance with the RINGS framework~\cite{coupette2025no}, which filters out datasets solvable without genuine structural reasoning by enforcing \emph{structural necessity} (P1) and \emph{feature--topology interdependence} (P2). \autoref{tab:datasets_overview} summarises the resulting set; full statistics, splits, and construction protocols are deferred to App.~\ref{app:datasets}.

\begin{table}[t]
\caption{Benchmark datasets and their roles in our evaluation. \emph{Predictive} indicates use for accuracy/F1 evaluation; \emph{Explainability} indicates availability of node-level ground-truth rationales for Jaccard@$k$ and AUROC. RINGS columns mark compliance with structural necessity (P1) and feature--topology interdependence (P2)~\cite{coupette2025no}. Edge feat.\ indicates whether the dataset carries discrete bond/edge features.}
\label{tab:datasets_overview}
\centering
\small
\begin{tabular}{lcccccc}
\toprule
\textbf{Dataset} & \textbf{Domain} & \textbf{Task} & \textbf{Edge feat.} & \textbf{Pred.} & \textbf{Expl.} & \textbf{RINGS} \\
\midrule
PATTERN     & SBM         & Node class.  & --   & \checkmark & --         & \checkmark / \checkmark \\
NCI1        & Chemistry   & Graph class. & --   & \checkmark & --         & \checkmark / \checkmark \\
MolHIV      & Chemistry   & Graph class. & \checkmark & \checkmark & --   & \checkmark / \checkmark \\
\midrule
BA-2Motif   & Synthetic   & Graph class. & --   & \checkmark & \checkmark & \checkmark / \checkmark \\
Di-Halo-Benzene & Chemistry & Graph class. & \checkmark & \checkmark & \checkmark & \checkmark / \checkmark \\
MNIST-75sp  & Vision      & Graph class. & --   & \checkmark & \checkmark & \checkmark / \checkmark \\
\bottomrule
\end{tabular}
\end{table}

\paragraph{Predictive-only benchmarks.}
\textbf{PATTERN}~\cite{dwivedi2020benchmarking} is a node classification task on stochastic block models that probes a model's ability to discern local topological structure in dense graphs.
\textbf{NCI1}~\cite{wale2008comparison} comprises 4{,}110 chemical compounds screened for activity against non-small-cell lung cancer.
\textbf{OGB-MolHIV}~\cite{wu2018moleculenet,hu2020open} extends the chemistry setting to a larger and more diverse molecular collection (HIV inhibition).
None of these datasets provides verified node-level rationales, so we use them solely to verify that B-cos GNNs retain competitive predictive capacity on standard tasks.

\paragraph{Benchmarks with verified rationales.}
\textbf{BA-2Motif}~\cite{ying2019gnnexplainer,luo2020parameterized} attaches one of two five-node motifs (House or Cycle) to a random Barab\'asi--Albert graph; the motif nodes form the exact rationale. We use a regenerated variant with strict node-level masks (App.~\ref{app:datasets}) and, given known instabilities of this benchmark~\cite{faber2021comparing,reproducing-pgexplainer}, complement it with the chemically and visually grounded rationales below.
\textbf{MNIST-75sp}~\cite{knyazev2019understanding} represents handwritten digits as superpixel graphs, with non-zero superpixels constituting the rationale. Following B-cos vision conventions~\cite{bohle2022b}, we augment the standard node features (normalized $(x,y)$ position, greyscale intensity $s$) with $1-s$ so that dark and bright intensities map to distinct input directions. We further sparsify the graphs by removing edges between superpixels whose normalized centers are farther than $\tau{=}0.1$ apart, restricting message passing to local neighborhoods (details in App.~\ref{app:datasets}).

\textbf{Di-Halo-Benzene-Isomer} is a benchmark we construct from ZINC 
scaffolds onto which a di-halo-benzene ring is grafted. The 9-class 
label combines halogen identity (Cl, F, Br) and ring position 
(ortho, meta, para), with an 8-node rationale (six aromatic carbons 
plus two halogens). Halogen identity is encoded in node features 
alone (3 classes), and positional information in topology alone 
(3 classes); only their joint use recovers all 9 classes, ensuring 
genuine structure–feature interdependence by construction.

\subsection{Baselines}
\label{sec:baselines}

We compare against three baselines spanning the main families of GNN 
explainability, i.e., an optimization-based post-hoc method, a gradient-based 
post-hoc method, and an inherently interpretable model.

\paragraph{Conventional GIN with post-hoc explainers.}
We train a GIN with conventional ReLU-MLPs of identical depth, width, 
and readout structure. Differences in trainable parameter count arise 
only from the bias-free formulation of B-cos transforms and the fixed 
$\epsilon = 0$ in B-cos GIN, both inherent to the construction. This isolates the effect of replacing ReLU-MLPs with B-cos 
transforms on both predictive and explanation quality. We pair this 
baseline with two post-hoc explainers: 
\textbf{GNNExplainer}~\cite{ying2019gnnexplainer}, which optimizes 
a soft edge mask via mutual information maximization, and 
\textbf{Integrated Gradients (IG)}~\cite{sundararajan2017axiomatic}, 
which accumulates gradients along an interpolation path from a 
zero-feature baseline. IG is the most direct gradient-based analogue 
to B-cos contribution maps and serves as our primary attribution 
baseline.

\paragraph{Inherently interpretable: GSAT.}
\textbf{GSAT}~\cite{miao2022interpretable} jointly trains a GNN with 
a stochastic attention mechanism under a graph information-bottleneck 
objective. 
We retrain GSAT with the same GIN/GINE backbones used 
throughout this paper, using the authors' reference implementation 
and default subgraph-extractor hyperparameters (App.~\ref{app:architectures}, 
Tabs.~\ref{tab:gsat_gin_compact},~\ref{tab:gsat_gine_summary}). 
We were unable to reproduce the predictive performance reported by 
\cite{miao2022interpretable} on all datasets under our shared GIN/GINE
backbone configuration. The reported numbers are from our retrained 
models for fair comparison.

\subsection{Metrics}
\label{sec:metrics}

\paragraph{Predictive performance.}
We report classification F1-score (and accuracy where standard for the 
dataset) over five seeds; per-dataset metric choices follow the 
benchmark conventions (Tab.~\ref{tab:datasets_overview}).

\paragraph{Explanation quality.}
For datasets with verified node-level ground-truth rationales 
$G^* \subseteq V$, we evaluate explanations as the alignment between 
$G^*$ and a node-importance scoring function $s: V \to \mathbb{R}$ 
produced by the explainer. 
We evaluate at the node level because B-cos contribution maps decompose the prediction into per-node contributions by construction and node-level evaluation matches this granularity directly. 
Also, we evaluate at fixed cardinality $k = |G^*|$, which makes the metric directly comparable across explainers and datasets. No threshold needs to be chosen, and the score reaches $1$ exactly 
when the predicted and ground-truth sets coincide.

\paragraph{Jaccard@$k$ (primary metric).}
Let $S_k \subseteq V$ denote the top-$k$ nodes under $s$. We define
\begin{equation}
    \label{eq:jaccard}
    \text{Jaccard@}k = \frac{|S_k \cap G^*|}{|S_k \cup G^*|}, 
    \qquad k = |G^*|.
\end{equation}
With $k = |G^*|$, Jaccard@$k$ penalises both false positives 
(spurious nodes) and false negatives (missing motif nodes), and 
reaches $1$ if and only if $S_k = G^*$.  

\paragraph{Node AUROC (secondary metric).}
For comparability with prior work, we additionally report a 
node-level AUROC:
\begin{equation}
    \label{eq:auroc}
    \text{AUROC} = \frac{1}{|G^*|\cdot|V \setminus G^*|} 
    \sum_{u \in G^*} \sum_{v \in V \setminus G^*} \mathbf{1}[s_u > s_v].
\end{equation}
A high AUROC indicates that the scorer ranks motif nodes above non-motif nodes on average, but does not require the top-$k$ nodes to coincide with the motif. We therefore treat Jaccard@$k$ as the more discriminating measure and AUROC as a comparability metric with prior work.

\paragraph{Limitations of structural metrics.}
Both metrics measure \emph{plausibility}, i.e. agreement with a known rationale, rather than \emph{faithfulness} to the model's computation. 
For B-cos GNNs, faithfulness is established by construction: predictions 
decompose exactly into the contributions reported as explanations 
(\autoref{sec:bcos_gnn}). For all baselines, faithfulness is not 
guaranteed, and recent   evaluations indicate that several popular post-hoc explainers perform no better than random 
edge ranking despite strong AUROC scores~\cite{amara2023ginxeval}. 
We adopt these standard ground-truth-based metrics for direct 
comparability with prior work and discuss possibilities for explicit 
faithfulness evaluation in \autoref{sec:conclusion}.

\subsection{Predictive performance}
\label{sec:predictive_perf}

Inherent interpretability typically incurs some cost in predictive 
accuracy. We quantify this cost on a diverse set of benchmarks 
(\autoref{sec:datasets}) by comparing B-cos GINs/GINEs against three 
references: conventional GIN/GINE backbones, the inherently interpretable 
GSAT-GIN/GINE baseline, and, where available, results published with 
the original benchmark.

\begin{table}[!htb]
\caption{
Predictive performance for explainability benchmarks (top 3 rows) and standard benchmarks (bottom 3 rows). Performance is measured in terms of F1-score.
B-cos GINs exhibit a small loss in predictive accuracy compared to conventional GINs but remain competitive.
}
\label{tab:pred_perf}
\vskip 0.15in
\begin{center}
\begin{small}
\begin{sc}
\begin{tabular}{llccc}
\toprule
& & \multicolumn{3}{c}{\textbf{F1}} \\
\cmidrule(lr){3-5}
\textbf{Dataset} & \textbf{Task} & \textbf{GIN(E)} & \textbf{B-cos GIN(E) ($B=2$)} & \textbf{GSAT-GIN(E)} \\
\midrule
BA-2Motif    & Graph Class. & \textbf{1.000 $\pm$ 0.00} & \underline{1.000 $\pm$ 0.00} & \underline{1.000 $\pm$ 0.00} \\
Di-Halo      & Graph Class. & \textbf{1.000 $\pm$ 0.00} & \underline{1.000 $\pm$ 0.00} & 0.97 $\pm$ 0.05 \\
MNIST-75SP   & Graph Class. & {0.86 $\pm$ 0.01} & \textbf{0.93 $\pm$ 0.01} & \underline{0.89 $\pm$ 0.01} \\
\midrule
PATTERN      & Node Class. & \textbf{0.69 $\pm$ 0.01} & \textbf{0.69 $\pm$ 0.02} & 0.43 $\pm$ 0.04 \\
NCI1         & Graph Class. & \textbf{0.80 $\pm$ 0.02} & 0.72 $\pm$ 0.01 & \underline{0.78 $\pm$ 0.02} \\
OGB-MolHIV  & Graph Class. & \textbf{0.79 $\pm$ 0.01} & 0.73 $\pm$ 0.03 & \underline{0.75 $\pm$ 0.02} \\
\bottomrule
\end{tabular}
\end{sc}
\end{small}
\end{center}
\vskip -0.1in
\end{table}

\autoref{tab:pred_perf} reports test F1 scores across five seeds for the B-cos model with alignment pressure $B=2$, identified as the most balanced configuration by our sensitivity analysis (\autoref{sub_sec:ablation_b}). On the two synthetic benchmarks (BA-2Motif, Di-Halo-Benzene) all models achieve near-perfect accuracy, in line with the well-defined nature of these tasks. Differences emerge on the remaining datasets. On PATTERN, B-cos GIN ties with the conventional GIN. On NCI1 and OGB-MolHIV, the conventional GIN/GINE leads, with B-cos and GSAT trailing by similar margins. On MNIST-75sp, by contrast, the B-cos GIN achieves a substantially higher average F1-score than both baselines.

Across all benchmarks, enforcing alignment incurs moderate (single-digit F1 points) predictive cost on two datasets and clearly improves performance on MNIST-75sp. This is mostly consistent with the trade-off 
documented for B-cos networks in vision and language~\cite{bohle2022b,
bohle2024b,wang2025b} and supports the view that alignment-enforcing 
architectures retain enough capacity to be competitive on tasks 
requiring genuine graph reasoning.

\subsection{Explainability performance}
\begin{table*}[ht!]
\caption{Detailed evaluation of Explanation Accuracy and computational cost (Mean $\pm$ Std over 5 seeds). \textbf{ms/graph} denotes the average inference time per explanation. B-cos achieves the best structural alignment on all three datasets.}
\label{tab:timing_and_accuracy}
\vskip 0.05in
\begin{center}
\begin{small}
\begin{sc}

\textbf{BA-2Motif (Synthetic); Size of test set = 100 Graphs} \\
\vskip 0.05in
\begin{tabular}{lccccc}
\toprule
\textbf{Method} & \textbf{Jacc.} $\uparrow$ & \textbf{AUC} $\uparrow$ & \textbf{ms/graph} & \textbf{Acc} & \textbf{F1} \\
\midrule
GNNExplainer & \underline{0.57}{\scriptsize$\pm$0.14} & 0.82{\scriptsize$\pm$0.09} & 290.92{\scriptsize$\pm$6.42} & 1.00 $\pm$0.00 & 1.00 $\pm$0.00\\
IG & \underline{0.54}{\scriptsize$\pm$0.12} & \underline{0.89}{\scriptsize$\pm$0.02} & 61.35{\scriptsize$\pm$0.91} & 1.00 $\pm$0.00 & 1.00 $\pm$0.00\\
GSAT & \underline{0.57}{\scriptsize$\pm$0.03} & 0.80{\scriptsize$\pm$0.03} & \textbf{0.44}{\scriptsize$\pm$0.02} & 1.00 $\pm$0.00 & 1.00 $\pm$0.00\\
\textbf{B-cos (Ours)} & \textbf{0.84}{\scriptsize$\pm$0.02} & \textbf{0.96}{\scriptsize$\pm$0.03} & \underline{0.62}{\scriptsize$\pm$0.04} & 1.00 $\pm$0.00 & 1.00 $\pm$0.00\\
\bottomrule
\end{tabular}
\vskip 0.1in
\textbf{Di-Halo-Benzene (Chemical Structure); Size of test set = 900 Graphs} \\
\vskip 0.05in
\begin{tabular}{lccccc}
\toprule
\textbf{Method} & \textbf{Jacc.} $\uparrow$ & \textbf{AUC} $\uparrow$ & \textbf{ms/graph} & \textbf{Acc} & \textbf{F1} \\
\midrule
GNNExplainer & 0.36{\scriptsize$\pm$0.11} & 0.69{\scriptsize$\pm$0.12} & 316.06{\scriptsize$\pm$67.2} & \textbf{1.00} $\pm$0.00 & \textbf{1.00} $\pm$0.00\\
IG & \underline{0.83}{\scriptsize$\pm$0.07} & \underline{0.98}{\scriptsize$\pm$0.01} & 70.51{\scriptsize$\pm$0.67} & \underline{0.97} $\pm$0.02 & \underline{0.97} $\pm$0.07\\
GSAT & 0.48{\scriptsize$\pm$0.00} & 0.86{\scriptsize$\pm$0.02} & \textbf{0.44}{\scriptsize$\pm$0.04} & \underline{0.97} $\pm$0.04 & \underline{0.97} $\pm$0.05\\
\textbf{B-cos (Ours)} & \textbf{0.96}{\scriptsize$\pm$0.01} & \textbf{0.99}{\scriptsize$\pm$0.01} & \underline{1.17}{\scriptsize$\pm$0.11} & \textbf{1.00} $\pm$0.00 & \textbf{1.00} $\pm$0.00\\
\bottomrule
\end{tabular}
\vskip 0.1in
\textbf{MNIST-75sp (Image Superpixels); Size of test set = 1000 Graphs} \\
\vskip 0.05in
\begin{tabular}{lccccc}
\toprule
\textbf{Method} & \textbf{Jacc.} $\uparrow$ & \textbf{AUC} $\uparrow$ & \textbf{ms/graph} & \textbf{Acc} & \textbf{F1} \\
\midrule
GNNExplainer & 0.34$\pm$0.03 & 0.64$\pm$0.05 & 327.03$\pm$6.53 &0.90$\pm$0.00 &0.89$\pm$0.04 \\
IG & \underline{0.63}$\pm$0.03 & \underline{0.92}$\pm$0.01 & 77.93$\pm$5.00 & \underline{0.90}$\pm$0.01 & \underline{0.90}$\pm$0.01\\
GSAT & 0.61$\pm$0.06 & 0.86$\pm$0.02 & \textbf{0.72}$\pm$0.05 &0.89$\pm$0.01&0.89$\pm$0.01 \\
\textbf{B-cos (Ours)} & \textbf{0.91}$\pm$0.01 & \textbf{0.99}$\pm$0.00 & \underline{0.94}$\pm$0.07 & \textbf{0.93}$\pm$0.01 & \textbf{0.93}$\pm$0.091\\
\bottomrule
\end{tabular}

\end{sc}
\end{small}
\end{center}
\end{table*}

\vskip -0.1in

Table~\ref{tab:timing_and_accuracy} reports explanation quality and average time taken to generate explanations across the three benchmarks with verified rationales. Predictive performances are re-reported in the F1-score column for convenient comparison.
The inherently explainable methods are both substantially faster than the post-hoc explainers, as they require either a single forward (GSAT) or forward and backward pass (B-cos). This discrepancy also explains why our method is marginally slower, than GSAT.
In terms of explanation quality, B-cos GINs achieve the highest Jaccard@k on all datasets ($0.84$ on BA-2Motif, $0.96$ on Di-Halo-Benzene, and $0.91$ on MNIST-75sp) with corresponding AUROC gains, demonstrating consistently superior structural alignment with ground-truth rationales.
This superiority on MNIST-75P is also visual upon qualitative comparison of explanations (see Figs. \ref{subfig:example_explanations}, \ref{fig:bcos_gsat_comp}).
Among the baselines, IG is the strongest with Jaccard@k of $0.83$ and $0.63$ on Di-Halo-Benzene and MNIST-75sp, where GNNExplainer and GSAT both fall short.
Overall, the weight-input alignment enforced by B-cos GINs demonstrates excellent suitability for identifying meaningful ground-truth rationales. 
This transparency benefit is practically free at inference time and may outweigh the modest predictive cost in applications with high interpretability requirements.

\subsection{Sensitivity to alignment pressure $B$}
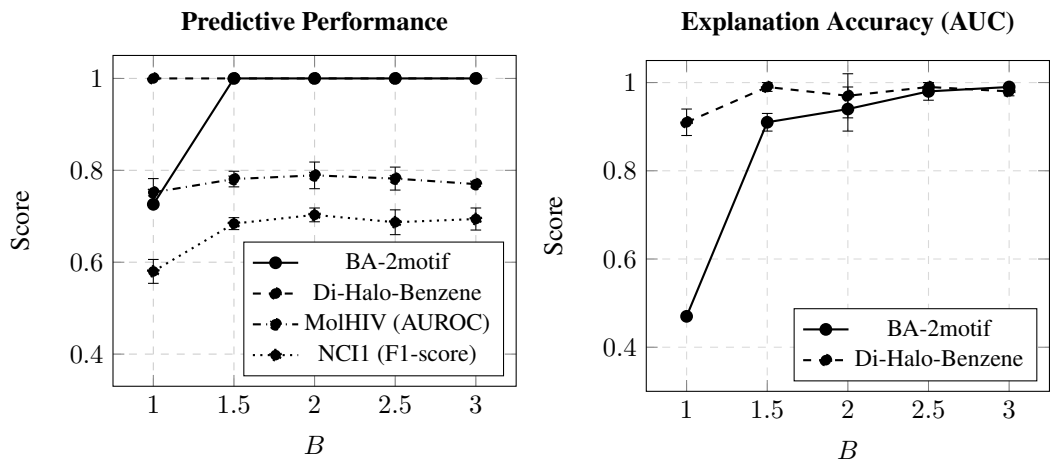
\begin{figure}[ht]
    \centering
    \begin{subfigure}[t]{0.495\textwidth}
        \label{subfig:b_predictive_performance}
        \centering
        \begin{tikzpicture}
            \begin{axis}[
                title={\textbf{Predictive Performance}},
                xlabel={$B$},
                ylabel={Score},
                xmin=0.75, xmax=3.25,
                ymin=0.33, ymax=1.05,
                xtick={1.0, 1.5, 2.0, 2.5, 3.0},
                ytick={0.4, 0.6, 0.8, 1.0},
                legend pos=south east,
                legend style={font=\small},
                grid=major,
                grid style={dashed, gray!40},
                width=\linewidth,
            ]
            
            \addplot[
                mark=*,
                solid,
                thick,
            ] coordinates {
                (1.0, 0.726)
                (1.5, 1.000)
                (2.0, 1.000)
                (2.5, 1.000)
                (3.0, 1.000)
            };
            \addlegendentry{BA-2motif}
            
            \addplot[
                mark=*,
                dashed,
                thick,
            ] coordinates {
                (1.0, 1.000)
                (1.5, 1.000)
                (2.0, 1.000)
                (2.5, 1.000)
                (3.0, 1.000)
            };
            \addlegendentry{Di-Halo-Benzene}

            \addplot[
                mark=*,
                dashdotted,
                thick,
                error bars/.cd,
                error bar style={solid},
                y dir=both, y explicit,
            ] coordinates {
                (1.0, 0.752) +- (0, 0.030)
                (1.5, 0.781) +- (0, 0.017)
                (2.0, 0.789) +- (0, 0.029)
                (2.5, 0.782) +- (0, 0.025)
                (3.0, 0.770) +- (0, 0.005)
            };
            \addlegendentry{MolHIV (AUROC)}

            \addplot[
                mark=*,
                dotted,
                thick,
                error bars/.cd,
                error bar style={solid},
                y dir=both, y explicit,
            ] coordinates {
                (1.0, 0.580) +- (0, 0.026)
                (1.5, 0.684) +- (0, 0.013)
                (2.0, 0.703) +- (0, 0.015)
                (2.5, 0.687) +- (0, 0.027)
                (3.0, 0.694) +- (0, 0.024)
            };
            \addlegendentry{NCI1 (F1-score)}
 
            \end{axis}
        \end{tikzpicture}
    \end{subfigure}
    \hfill
    \begin{subfigure}[t]{0.495\textwidth}
        \label{subfig:b_explanation_auc}
        \centering
        \begin{tikzpicture}
            \begin{axis}[
                title={\textbf{Explanation Accuracy (AUC)}},
                xlabel={$B$},
                ylabel={Score},
                xmin=0.75, xmax=3.25,
                ymin=0.3, ymax=1.05,
                xtick={1.0, 1.5, 2.0, 2.5, 3.0},
                ytick={0.4, 0.6, 0.8, 1.0},
                legend pos=south east,
                legend style={font=\small},
                grid=major,
                grid style={dashed, gray!30},
                width=\linewidth,
            ]
            \addplot[
                mark=*,
                solid,
                thick,
                error bars/.cd,
                error bar style={solid},
                y dir=both, y explicit,
            ] coordinates {
                (1.0, 0.47) +- (0, 0.00)
                (1.5, 0.91) +- (0, 0.02)
                (2.0, 0.94) +- (0, 0.05)
                (2.5, 0.98) +- (0, 0.02)
                (3.0, 0.99) +- (0, 0.00)
            };
            \addlegendentry{BA-2motif}
 
            \addplot[
                mark=*,
                dashed,
                thick,
                error bars/.cd,
                error bar style={solid},
                y dir=both, y explicit,
            ] coordinates {
                (1.0, 0.91) +- (0, 0.03)
                (1.5, 0.99) +- (0, 0.01)
                (2.0, 0.97) +- (0, 0.05)
                (2.5, 0.99) +- (0, 0.01)
                (3.0, 0.98) +- (0, 0.01)
            };
            \addlegendentry{Di-Halo-Benzene}
 
            \end{axis}
        \end{tikzpicture}
    \end{subfigure}

    \caption{
        Prediction Accuracy and Explanation AUC of a B-cos GIN model as a function of $B$.
        Errorbars indicate standard deviation over five independent runs with different random seeds.
    }
    \label{fig:b_sensitivity}
\end{figure}

\label{sub_sec:ablation_b}
We analyze the effect of varying the hyperparameter $B$, which, when increased, promotes stronger weight-input alignment. 
This is known to trade off potentially worse accuracy for higher interpretability \cite{bohle2022b}.
To investigate this relation for graph data and models, we train a B-cos GIN on several of our benchmark datasets and plot validation accuracy and explanation AUC as a function of $B \in \{1.0, 1.5, 2.0, 2.5, 3.0\}$ in \autoref{fig:b_sensitivity}.

The node update functions in a B-cos GIN with $B=1.0$ are effectively linear, so the corresponding subpar model predictive performances on three out of four datasets can be attributed to lower expressivity.
Endowing the models with non-linear capabilities by increasing alignment pressure towards $B = 2.0$ steadily improves performance, \eg, $0.581 \to 0.703$ F1-score on NCI1 and $0.752 \to 0.791$ AUROC on MolHIV, whereas further increases slightly decrease performance on these real-world datasets.
The explanation AUC on synthetic benchmarks with ground truth explanations increases steadily, reaching values of around $95\%$ for $B = 2.0$ and approaches $100\%$ for $B = 3.0$, which demonstrates that enforcing higher weight-input alignment can guide models toward better capturing the true underlying discriminatory motifs in the data.

\paragraph{Limitations.}
Our analysis is restricted to sum-aggregation backbones, since dynamic 
linearity requires linear aggregation combined with B-cos-only message 
and update functions.
It is important to note that models relying on weight-input alignment are suitable for classification tasks.
Exploring how to extend them to graph regression tasks and usage of mean, or attention-based aggregators is left to future work. 
Furthermore, our edge-feature extension treats edge attributes as additive biases on messages, so the resulting contribution maps attribute predictions to nodes and node features but not to edge attributes; deriving complementary edge-level contributions is left to future work.
Enforcing weight--input alignment incurs a moderate predictive cost on real-world 
benchmarks (\autoref{sec:predictive_perf}), consistent with the trade-off 
documented for B-cos networks in vision and 
language~\cite{bohle2022b,bohle2024b,wang2025b}. 
Finally, we could not reproduce the predictive performance reported for GSAT 
in~\cite{miao2022interpretable} under our shared backbone configuration; comparisons in this paper therefore rely on our retrained models.

\section{Conclusion}
\label{sec:conclusion}
We showed that B-cos transforms transfer to graphs essentially without much modification: composing them with sum-based aggregation preserves the alignment between weights and inputs that makes B-cos contribution maps interpretable in the first place, so the resulting GNN admits faithful per-node, per-feature contribution maps from a single backward pass. The construction is deliberately minimal (no auxiliary explainer, no modified objective, etc.) and introduces a single hyperparameter $B$ trading off alignment strength against predictive flexibility.
Our experiments show that B-cos GINs and GINEs outperform both post-hoc explainers (GNNExplainer, IG) and an inherent baseline (GSAT) in structural alignment with ground-truth motifs, generate explanations orders of magnitude faster than post-hoc methods, and incur only a modest accuracy cost on graph classification datasets. Our analysis covers sum aggregation. Extending the alignment argument 
to mean, max, and attention-based aggregators, deriving edge-level 
contributions complementary to our per-node decomposition, complementing 
ground-truth-based metrics with in-distribution faithfulness 
evaluation~\cite{amara2023ginxeval}, and applying B-cos GNNs to larger 
real-world graphs where ground-truth rationales are typically unavailable are the natural next steps.
Broader impact of our work is discussed in \autoref{app:impact}.

\bibliographystyle{abbrv}
\bibliography{bibliography}

\newpage
\appendix
\section{Broader Impact}
\label{app:impact}
We facilitate improved interpretability for GNNs, which can positively impact application outcomes through increased transparency and data understanding beyond pure prediction. 
At the same time, interpretable models may be more susceptible to adversarial attacks (reverse engineering) and high quality explanations may provide a false sense of security even when predicitve accuracy is stil insufficient for safe application.
\section{Sum-aggregation GNNs preserve dynamic linearity}
\label{app:dynlin_proofs}

\begin{lemma}
\label{lemma1}
Let $f_{\theta_1}: \mathbb{R}^q \to \mathbb{R}^r$ and $g_{\theta_2}: \mathbb{R}^p \to \mathbb{R}^q$
be functions that satisfy dynamic linearity, let $\alpha_1, \ldots, \alpha_n \in \mathbb{R}$ be
scalar-valued, and $\mathbf{x}_1, \ldots, \mathbf{x}_n \in \mathbb{R}^p$ be vector-valued inputs.
Then
\[
  h_{(\theta_1,\theta_2)}(\tilde{\mathbf x}, \alpha)
  \;\doteq\;
  f_{\theta_1}\!\left(\sum_{j=1}^n \alpha_j\, g_{\theta_2}(\mathbf{x}_j)\right)
\]
also satisfies dynamic linearity with respect to the stacked input $\tilde{\mathbf{x}} = [\mathbf{x}_1^\top, \dots, \mathbf{x}_n^\top]^\top \in \mathbb R^{np}$ and input scalars $\alpha$.
\end{lemma}

\begin{proof}
Because $g_{\theta_2}$ satisfies dynamic linearity, for each $j$ we have
$g_{\theta_2}(\mathbf{x}_j) = \mathbf{W}_{\theta_2}(\mathbf{x}_j)\,\mathbf{x}_j$.
Scaling by $\alpha_j$ preserves this form since
$\alpha_j\,g_{\theta_2}(\mathbf{x}_j) = (\alpha_j\,\mathbf{W}_{\theta_2}(\mathbf{x}_j))\,\mathbf{x}_j$.
Summing over $j = 1, \ldots, n$ yields
\begin{equation}
  \label{eq:sum_s}
  \mathbf{s} \;\doteq\; \sum_{j=1}^n \alpha_j\,g_{\theta_2}(\mathbf{x}_j)
  = \sum_{j=1}^n \alpha_j\,\mathbf{W}_{\theta_2}(\mathbf{x}_j)\,\mathbf{x}_j.
\end{equation}
Since $f_{\theta_1}$ satisfies dynamic linearity,
$f_{\theta_1}(\mathbf{s}) = \mathbf{W}_{\theta_1}(\mathbf{s})\,\mathbf{s}$.
Substituting~\eqref{eq:sum_s} gives
\[
  h_{(\theta_1,\theta_2)}(\tilde{\mathbf x}, \alpha)
  = \mathbf{W}_{\theta_1}(\mathbf{s})
    \sum_{j=1}^n \alpha_j\,\mathbf{W}_{\theta_2}(\mathbf{x}_j)\,\mathbf{x}_j
  = \mathbf{W}_{(\theta_1,\theta_2)}(\tilde{\mathbf{x}})\,\tilde{\mathbf{x}},
\]
where
\[
  \mathbf{W}_{(\theta_1,\theta_2)}(\tilde{\mathbf{x}})
  \;\doteq\;
  \underbrace{\mathbf{W}_{\theta_1}(\mathbf{s})}_{\in \mathbb{R}^{r \times q}}\,
  \underbrace{\bigl[\alpha_1\,\mathbf{W}_{\theta_2}(\mathbf{x}_1)
        \;\cdots\;
        \alpha_n\,\mathbf{W}_{\theta_2}(\mathbf{x}_n)\bigr]}_{\in \mathbb{R}^{q \times np}}
  \;\in\; \mathbb{R}^{r \times np}.
\]

This matrix depends only on $\tilde{\mathbf{x}}$, $\theta_1$, $\theta_2$, and the scalars
$\alpha$. 
Hence $h_{(\theta_1,\theta_2)}$ satisfies dynamic linearity.
\end{proof}

\begin{proposition}
\label{prop1}
If all $\Phi^{(k)}$ and $\Psi^{(k)}$ ($k = 1, \ldots, L$) in a message-passing GNN (\ref{eq:message_passing}) satisfy dynamic linearity and the aggregation is carried out exclusively via weighted summation with weights given by $\mathbf{A}$, then for every node $i$ and depth $k \in \{0, \ldots, L\}$ the node embeddings satisfy
\[
  \mathbf{x}^{(k)}_i
  = \mathbf{W}^{(k)}_i(\mathbf{X}, \mathbf{A})\,\mathbf{x}\quad\text{and}\quad
  \mathbf x^{(k)} = \mathbf{W}^{(k)}(\mathbf{X}, \mathbf{A})\,\mathbf{x}
\]
where $\mathbf{x} \in \mathbb{R}^{np_0}, \mathbf{x}^{(k)} \in \mathbb{R}^{np_k},$ are the stacked node features and embeddings, $\mathbf{W}^{(k)}_i(\mathbf X, \mathbf A) \in \mathbb{R}^{p_k \times np_0}$ depends only on the graph input $(\mathbf{X}, \mathbf{A}$) and the GNN parameters until layer $k$, and 
$\mathbf{W}^{(k)}(\mathbf{X}, \mathbf{A}) \doteq [(\mathbf{W}^{(k)}_1(\mathbf{X}, \mathbf{A}))^T \ldots (\mathbf{W}^{(k)}_n(\mathbf{X}, \mathbf{A}))^T]^T$.
\end{proposition}

\begin{proof}
We proceed by induction on the depth $k$. 

\textbf{Base Case ($k=0$):} 
The property holds trivially for $\mathbf{x}^{(0)}_i$. We can define $\mathbf{W}^{(0)}_i = \mathbf{E}_i$, where $\mathbf{E}_i$ is the block-selector matrix such that $\mathbf{E}_i \mathbf{x} = \mathbf{x}_i$.

\textbf{Inductive Step:} 
Suppose $\mathbf{x}^{(k)}_j = \mathbf{W}^{(k)}_j(\mathbf{X},\mathbf{A})\,\mathbf{x}$ for all nodes $j$ at depth $k$. Using $\mathbf a_i$ to denote the $i$-th row of $\mathbf A$, at layer $k+1$, the aggregated message is:
\[
  \mathbf{z}_i^{(k+1)}
  = \sum_{j=1}^{n} \mathbf a_{ij}\,\Phi^{(k+1)}\!\left(\mathbf{x}_j^{(k)}\right).
\]
By Lemma~\ref{lemma1} (using scalar weights $\alpha = {\mathbf a}_i$, $g_{\theta_2} = \Phi^{(k+1)}$ and the identity function for $f_{\theta_1}$), 
$\mathbf{z}_i^{(k+1)}$ is dynamically linear in $\mathbf x^{(k)}$ with weights denoted by $\mathbf W^{(k+1)}_\Phi(\mathbf x^{(k)}, \mathbf A)$ that only depend on $\mathbf x^{(k)}$, $\mathbf A$ ($\mathbf a_i$) and the parameters of $\Phi^{(k+1)}$.
The previous node embeddings, $\mathbf{x}^{(k)}$, are also linear in $\mathbf{x}$ by the inductive hypothesis, thus $\mathbf{z}_i^{(k+1)}$ is also linear in $\mathbf{x}$, \ie,
\begin{align*}
    \mathbf{z}_i^{(k+1)} &= \mathbf W^{(k+1)}_\Phi(\mathbf x^{(k)}, \mathbf A) \mathbf x^{(k)} \\
     &= 
     \underbrace{\mathbf W^{(k+1)}_\Phi(\mathbf W^{(k)}(\mathbf X, \mathbf A) \mathbf x, \mathbf A) \mathbf W^{(k)}(\mathbf X, \mathbf A)}_{= \mathbf W^{(k+1)}_z(\mathbf X, \mathbf A)}
     \mathbf x
\end{align*}

with a weight matrix $\mathbf W^{(k+1)}_z(\mathbf X, \mathbf A)$ depending only on $\mathbf{X}$ and $\mathbf{A}$.

The node update is then:
\[
  \mathbf{x}_i^{(k+1)}
  = \Psi^{(k+1)}\!\left(\mathbf{x}_i^{(k)},\;\mathbf{z}_i^{(k+1)}\right).
\]
Since $\Psi^{(k+1)}$ satisfies dynamic linearity, it can be expressed as:
\[
  \mathbf{x}_i^{(k+1)} = \mathbf{W}_{\Psi}^{(k+1)}\!\left(\mathbf{x}_i^{(k)},\,\mathbf{z}_i^{(k+1)}\right)
    \begin{pmatrix}\mathbf{x}_i^{(k)} \\ \mathbf{z}_i^{(k+1)}\end{pmatrix}.
\]
Substituting the expressions for $\mathbf{x}_i^{(k)}$ and $\mathbf{z}_i^{(k+1)}$ (both of which are linear in $\mathbf{x}$), we see that $\mathbf{x}_i^{(k+1)}$ is likewise linear in $\mathbf{x}$. Stacking the coefficient matrices yields $\mathbf{W}^{(k+1)}_i(\mathbf{X},\mathbf{A})$, which depends only on $\mathbf{X}$ and $\mathbf{A}$ through the intermediate representations. This completes the induction.
\end{proof}

\paragraph{Dynamic linearity of readout}
By Lemma~\ref{lemma1} (with $\alpha_i = 1$ for all $i$, $f_{\theta_1} = \Theta$ and $g_{\theta_2}$ as the identity function) we immediately see that GNN predictions obtained through global sum pooling
followed by any dynamically linear readout module $\Theta$ also satisfy
dynamic linearity.

\section{RINGS Framework Principles}
\label{app:rings}

To ensure that our evaluation measures genuine graph reasoning capabilities rather than the identification of spurious correlations, we strictly adhere to the RINGS framework \cite{coupette2025no}. RINGS argues that many standard datasets in Graph XAI fail to test structural reasoning because they can be solved by "trivial" models (e.g., DeepSets or MLPs) that ignore topology entirely.

We filter our dataset selection to satisfy two critical principles:
\begin{itemize}
    \item \textbf{P1 (Structural Necessity):} The task label must strictly depend on the graph topology (e.g., specific motifs or ring structures) rather than being solvable by node features alone. This disqualifies datasets where a simple MLP could achieve high accuracy, as explanations in such settings often devolve into feature importance maps rather than structural rationales.
    \item \textbf{P2 (Feature-Topology Interdependence):} The node features and graph structure must provide complementary information. The GNN must effectively combine both to solve the task (e.g., identifying \textit{specific} atoms in a \textit{specific} arrangement).
\end{itemize}

By prioritizing benchmarks that satisfy these conditions (such as the Di-Halo-Benzene isomers and BA-2Motif), we ensure that the ground-truth rationale ($G_S$) contains the minimal information required for classification, while the background graph serves as a true confounding factor.

\section{Dataset Details}
\label{app:datasets}

We provide detailed statistics and construction protocols for the datasets used in our evaluation. The summary of dataset statistics is provided in Table \ref{tab:dataset_stats}.

\begin{table}[h]
\centering
\caption{Summary statistics for all benchmark datasets. \textbf{Avg. Nodes} and \textbf{Avg. Edges} refer to the mean number of nodes and edges per graph.}
\label{tab:dataset_stats}
\begin{small}
\begin{sc}
\begin{tabular}{lcccccc}
\toprule
\textbf{Dataset} & \textbf{Task Type} & \textbf{Graphs} & \textbf{Classes} & \textbf{Avg. Nodes} & \textbf{Avg. Edges} & \textbf{Metric} \\
\midrule
\multicolumn{7}{l}{\textit{Predictive Performance Benchmarks}} \\
PATTERN & Node Class. & 14,000 & 2 & 118.9 & 4,836.2 & Acc/F1 \\
NCI1 & Graph Class. & 4,110 & 2 & 29.9 & 32.3 & Acc/F1 \\
\midrule
\multicolumn{7}{l}{\textit{Explainability Benchmarks (Semi-Synthetic)}} \\
BA-2Motif (Custom) & Graph Class. & 1,000 & 2 & 23.70 & 48.40 & Acc/F1/Jaccard/AUC \\
Di-Halo-Benzene & Graph Class. & 9000 & 9 & 51.65 & 108.79 & Acc/F1/Jaccard/AUC \\
\bottomrule
\end{tabular}
\end{sc}
\end{small}
\end{table}

\subsection{Standard Benchmarks}

\textbf{PATTERN} \cite{dwivedi2020benchmarking} is a large-scale node classification dataset generated via Stochastic Block Models (SBM). The graphs simulate communities with varying intra- and inter-community connection probabilities. The task is to classify nodes based on whether they belong to a specific embedded pattern ($\mathcal{P}$) or the background ($\mathcal{G}$). The dataset contains 10,000 training, 2,000 validation, and 2,000 test graphs. We use this dataset to test the model's ability to discern local topological structures in dense graphs.

\textbf{NCI1} \cite{wale2008comparison} is a biological dataset provided by the National Cancer Institute, consisting of 4,110 chemical compounds screened for activity against non-small cell lung cancer. The dataset is balanced between active and inactive compounds.  

\textbf{MNIST-75sp Feature Encoding}\\
We extend the standard MNIST-75sp encoding to four channels per superpixel 
to follow the B-cos convention of complementary feature encodings 
\cite{bohle2022b}: (i) mean greyscale intensity in $[0,1]$; (ii) 
$1 - \text{greyscale}$; (iii) and (iv) normalized $(x, y)$ centroid 
position in $[0,1]^2$. This four-channel encoding is used identically 
across B-cos GIN, conventional GIN, and GSAT-GIN to ensure fair 
comparison.

\subsection{Custom Explainability Benchmarks}

To enable rigorous quantitative evaluation of explanation fidelity, we constructed two semi-synthetic datasets where the "ground truth" rationale is strictly defined.

\textbf{Custom BA-2Motif} \\
We generated a variant of the BA-2Motif dataset to ensure precise node-level ground truth.
\begin{itemize}
    \item \textbf{Base Graphs:} Barabási-Albert (BA) graphs serving as the random background structure.
    \item \textbf{Motifs:} We insert one of two 5-node motifs: a "House" (class 0) or a "Cycle" (class 1).
    \item \textbf{Features:} Node features are one-hot encoded degree vectors. Edge features are not used.
    \item \textbf{Ground Truth:} A binary mask is generated for every graph, labeling the 5 motif nodes as 1 and all background nodes as 0. This allows us to compute the exact Intersection-over-Union (Jaccard) between the model's top-5 explanations and the true motif.
\end{itemize}

\textbf{Di-Halo-Benzene-Isomer} \\
This dataset was designed to test the model's sensitivity to geometric isomers and edge attributes in a chemical context.
\begin{itemize}
    \item \textbf{Construction:} We extracted random molecular scaffolds from the ZINC database to serve as realistic background structures. On each scaffold, we grafted a di-halo-benzene ring.
    \item \textbf{Classes (9 total):} The class label is determined by the combination of the halogen type (Chlorine, Fluorine, Bromine) and its positional isomerism on the benzene ring (Ortho, Meta, Para).
    \item \textbf{Constraints:} 
    \begin{itemize}
        \item Node features: One-hot encoded atom type (C, N, O, F, Cl, Br, etc.).
        \item Edge features: One-hot bond type (Single, Double, Triple, Aromatic).
    \end{itemize}
    \item \textbf{Ground Truth:} The explanation mask includes the 6 carbon atoms of the benzene ring and the 2 attached halogen atoms. The challenge for the model is to ignore the potentially large and scaffold but correctly identify the ring and the precise distance (number of hops) between the halogens to distinguish isomers (e.g., Ortho vs. Meta).
\end{itemize}

\section{Qualitative analysis}
\label{app:explanation_examples}
The quantitative superiority of B-cos GNNs is further supported by visual inspection of the generated contribution maps.

\textbf{MNIST-75P:}
We display one randomly sampled data point together with ground truth, GSAT and B-cos GIN explanatory subgraphs in Fig.~\ref{fig:bcos_gsat_comp}.

\textbf{Topological Isolation (BA-2Motif):}
Figure \ref{fig:ba2motif_viz} presents a direct comparison between the ground-truth motif and the B-cos explanation over a sample of the BA-2motif dataset.
As illustrated, the model prioritizes the nodes constituting the "House" and "Cycle" structures (red nodes), while substantially attenuating the influence of the random attachment graph. 

\begin{figure*}[t!]
\begin{center}
\centerline{\includegraphics[width=\textwidth]{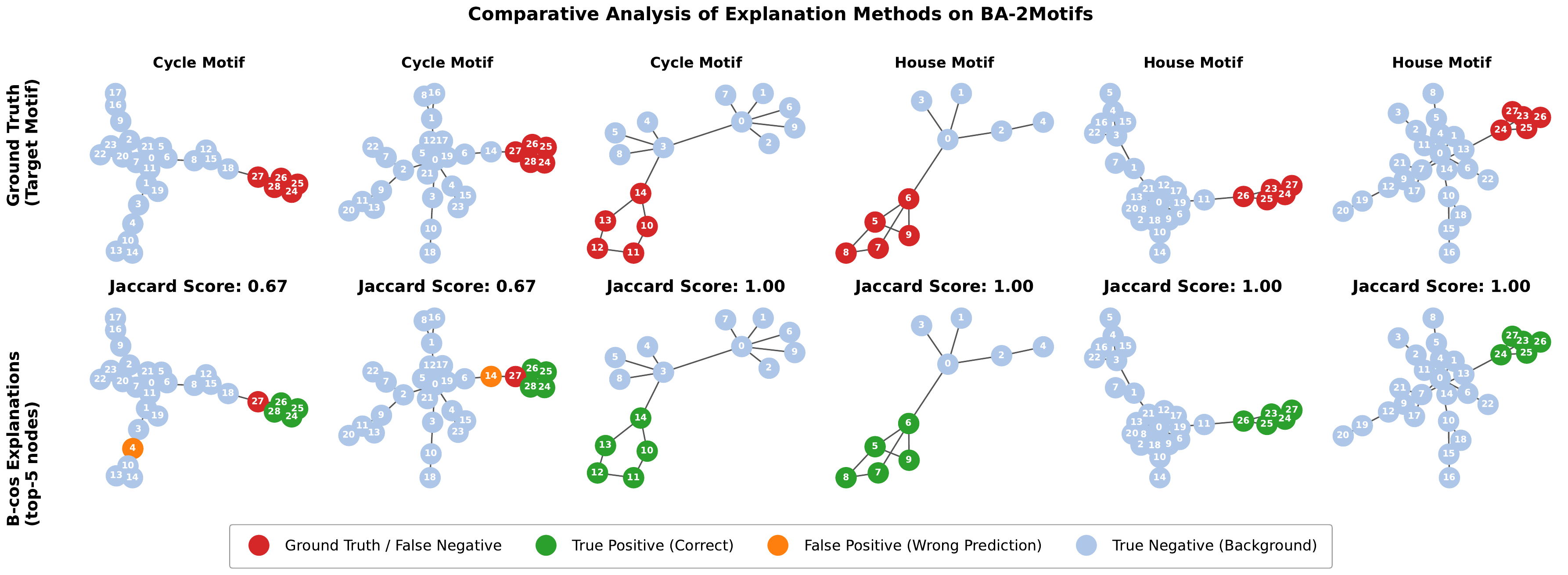}}
\caption{Visualizing Ground Truth rationales (Top row) against the learned B-cos Explanations (bottom row) of Di-Halo-Benzene-Isomer dataset. The Jaccard score between the ground-truth motif nodes (marked in blue) and the top-k nodes identified by the B-cos GINE model is shown above each molecular graph, along with the predicted label at the bottom.}
\label{fig:ba2motif_viz}
\end{center}
\vskip -0.2in
\end{figure*}

\textbf{Chemical Precision (Di-Halo-Benzene):}
Figure \ref{fig:dihalo_viz} demonstrates the model's ability to discern subtle geometric isomers. 
Classification in this dataset requires identifying the relative positions (ortho, meta, para) of specific halogen atoms(Cl, F, Br). 
The B-cos contribution maps consistently identify the benzene ring and the constituent chlorine, fluorine, or bromine atoms, while assigning significantly lower importance to the complex molecular scaffold. 
Crucially, the model distinguishes between isomers (e.g., \textit{di-chloro-ortho} vs. \textit{di-chloro-meta}) by attributing high importance to the relevant carbon atoms on the ring.
\begin{figure*}[t!]
\begin{center}
\centerline{\includegraphics[width=\textwidth]{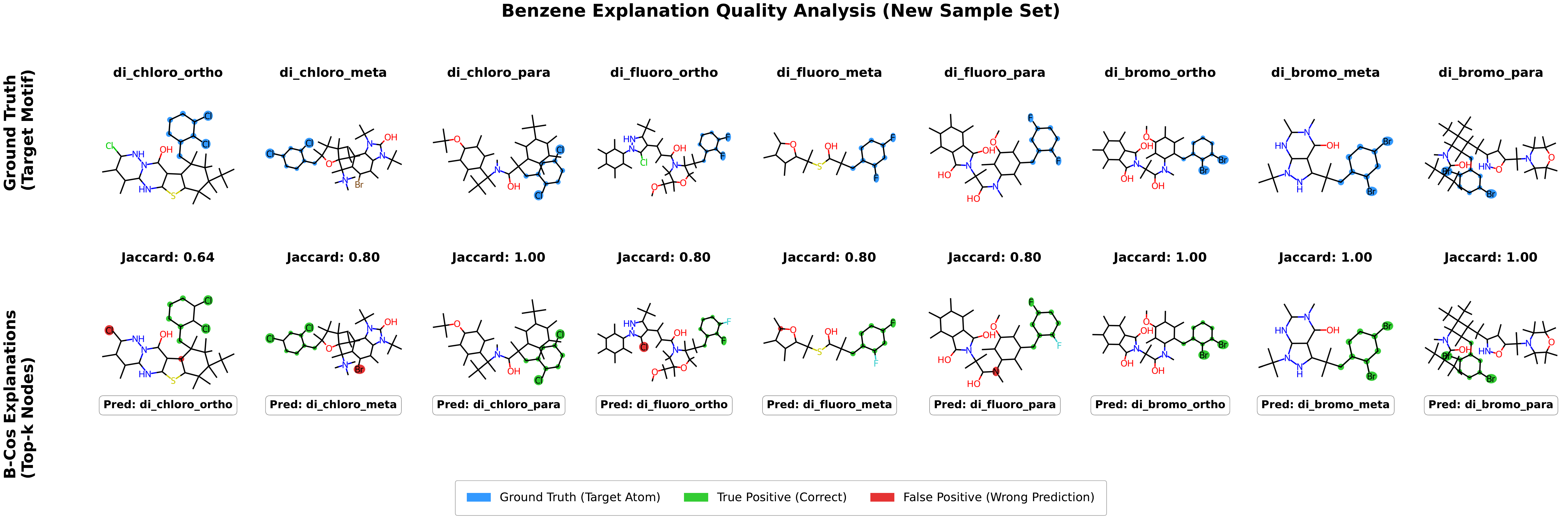}}
\caption{Visualizing Ground Truth rationales (Top row) against the learned B-cos Explanations (bottom row) of Di-Halo-Benzene-Isomer dataset. The Jaccard score between the ground-truth motif nodes (marked in blue) and the top-k nodes identified by the B-cos GINE model is shown above each molecular graph, along with the predicted label at the bottom.}
\label{fig:dihalo_viz}
\end{center}
\vskip -0.2in
\end{figure*}

\begin{figure*}[t!]
\begin{center}
\centerline{\includegraphics[width=\textwidth]{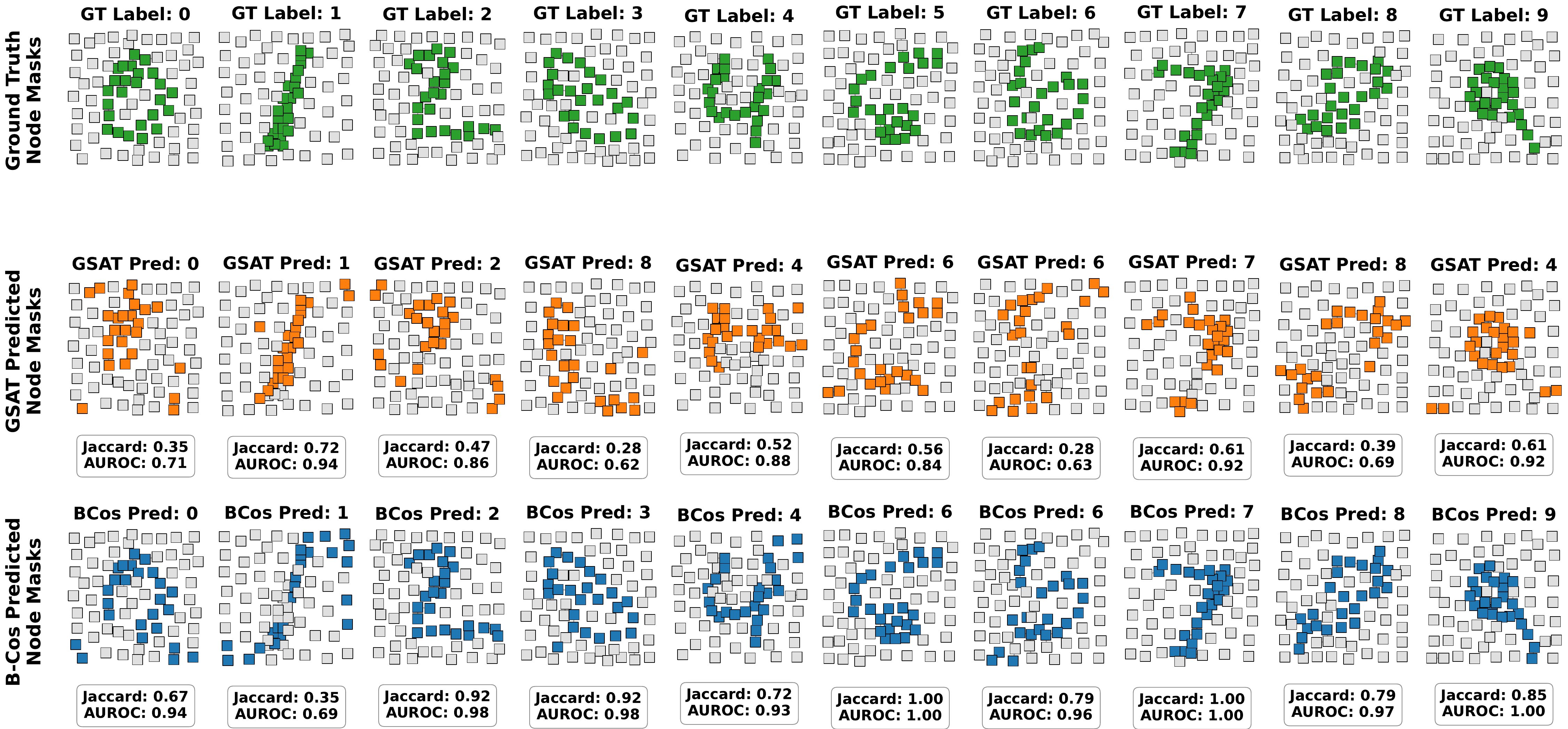}}
\caption{Qualitative comparison of explanation masks on the MNIST-75sp dataset. The top row displays the ground-truth node masks defining the structural rationale for digits 0 through 9. The middle and bottom rows illustrate the predicted explanation masks generated by the GSAT baseline and our proposed B-cos GNN, respectively. The Jaccard score and AUROC between the top-k predicted nodes and the ground-truth motif (highlighted in blue for B-cos) are provided below each prediction, alongside the predicted class label.}
\label{fig:bcos_gsat_comp}
\end{center}
\vskip -0.2in
\end{figure*}

\section{Model Architectures and Hyperparameters}
\label{app:architectures}

In this section, we detail the specific architectures and hyperparameter settings used for the experiments. We describe the models as instantiations of the general message-passing framework (Eq. \ref{eq:message_passing}) defined in the \textbf{Background}. To ensure a rigorous comparison, we standardize the readout mechanism across all experiments, distinguishing the models solely by their core transformation logic: the proposed B-cos architectures enforce dynamic linearity (Eq. \ref{eq:input_dependent_linear_transform}), whereas the "Vanilla" baselines utilize standard ReLU non-linearities. To prevent overfitting and ensure stable convergence, we employed early stopping alongside dynamic learning rate adjustment. The optimization process began with a learning rate of $10^{-3}$, which was decayed by a factor of $0.5$ whenever the validation metric stagnated for 25 epochs, down to a minimum of $10^{-6}$. Training was terminated if the validation performance failed to improve within a 25-epoch patience window.

\subsection{GIN Architectures (BA-2Motif, PATTERN \& NCI1)}
\label{app:gin_architectures}

For the BA-2Motif, PATTERN, and NCI1 datasets, we employ \textbf{Graph Isomorphism Network (GIN)} architectures. As discussed in Section \ref{sec:bcos_gnn}, GIN is selected as the backbone because its sum-aggregation preserves the linear path required for faithful B-cos contribution maps.

\textbf{B-cos GIN Formulation} \\
The B-cos GIN implements the update rule defined in Eq. \ref{eq:bcos_gin_update} of the Method section. In the context of the general MPNN notation (Eq. \ref{eq:message_passing}), the B-cos GIN defines the update function $\Psi^{(k)}$ as a \textbf{B-cos MLP} and the message function $\Phi^{(k)}$ as the identity:
\begin{equation}
    \mathbf{h}_i^{(k)} = \Psi^{(k)}_{\text{Bcos}} \left( (1 + \epsilon) \mathbf{h}_i^{(k-1)} + \sum_{j \in \mathcal{N}(i)} \mathbf{h}_j^{(k-1)} \right).
\end{equation}
Here, $\Psi^{(k)}_{\text{Bcos}}$ consists of stacked B-cos transformations (Eq. \ref{eq:bcos}), ensuring that the mapping from input features to output remains dynamically linear. This allows for the exact computation of node contribution maps $\mathbf{M}$ as derived in Eq. \ref{eq:gnn_contribution_map}.

\textbf{Vanilla GIN Formulation} \\
The baseline Vanilla GIN uses the same aggregation structure but employs standard non-linearities that do not satisfy the B-cos properties. The update function $\Psi^{(k)}$ is a standard Multi-Layer Perceptron (MLP) with ReLU activations:
\begin{equation}
    \mathbf{h}_i^{(k)} = \text{MLP}^{(k)}_{\text{ReLU}} \left( (1 + \epsilon) \mathbf{h}_i^{(k-1)} + \sum_{j \in \mathcal{N}(i)} \mathbf{h}_j^{(k-1)} \right).
\end{equation}

\textbf{Readout Strategy (Readout-Then-Agg)} \\
To maintain consistency and support node-based interpretability, both the B-cos model and the Vanilla baseline employ a \textbf{Readout-Then-Aggregate} strategy. The readout module $\Theta$ is applied in a row-wise fashion to the final node representations $\mathbf{x}^{(L)}$ \textit{before} global pooling:
\begin{equation}
    \mathbf{\hat{y}} = \sum_{i \in \mathcal{V}} \Theta(\mathbf{x}_i^{(L)})
\end{equation}
For the B-cos GNN, $\Theta$ is a B-cos MLP, enabling the exact calculation of node-specific weights $\mathbf{w}_{c,i}$ (Section \ref{sec:bcos_gnn}). For the Vanilla baseline, $\Theta$ is a standard ReLU-MLP. This unifies the architectural depth and pooling logic, ensuring that performance differences stem strictly from the choice of non-linearity (B-cos vs. ReLU).

\begin{table}[h]
\centering
\caption{Hyperparameter comparison for GIN architectures (BA-2Motif, PATTERN, NCI1).}
\label{tab:gin_comparison}
\begin{small}
\begin{sc}
\begin{tabular}{lcc}
\toprule
\textbf{Parameter} & \textbf{B-cos GIN} & \textbf{Vanilla GIN} \\
\midrule
\textbf{Backbone} & & \\
Conv Layer Type & B-cos GIN (Eq. \ref{eq:bcos_gin_update}) & GINConv \\
Hidden Channels & 64 & 64 \\
Num Layers ($L$) & 3 & 3 \\
$\epsilon$ (eps) & 0.0 (fixed) & Trainable \\
\midrule
\textbf{Readout} & & \\
Strategy & \textbf{Readout-Then-Agg} & \textbf{Readout-Then-Agg} \\
Pooling & Sum & Sum \\
Classifier $\Theta$ & 3-Layer B-cos MLP & 3-Layer ReLU MLP \\
\midrule
\textbf{Training} & & \\
Dropout & 0.0 & 0.0 \\
Optimizer & Adam & Adam \\
$B$ Parameter & 2.0 & N/A \\
\bottomrule
\end{tabular}
\end{sc}
\end{small}
\end{table}

\begin{table}[ht]
\centering
\caption{Architecture and hyperparameters: GSAT + GIN (Compact summary)}
\label{tab:gsat_gin_compact}
\renewcommand{\arraystretch}{1.15} 
\begin{tabular}{@{}p{0.38\linewidth} p{0.58\linewidth}@{}}
\toprule
\textbf{\textsc{Parameter}} & \textbf{\textsc{Specification}} \\
\midrule
\textbf{\textsc{Backbone}} & \\
\quad \textsc{Type} & Vanilla GIN (Maskable \texttt{GINConv}) \\
\quad \textsc{Conv Stack} & 4 Layers (\texttt{aggr='add'}) \\
\quad \textsc{Hidden Channels} & 64 \\

\midrule
\textbf{\textsc{Readout}} & \\
\quad \textsc{Pooling} & \texttt{SumAggregation} \\
\quad \textsc{Classifier} $\Theta$ & 2-Layer MLP ($H \rightarrow H/2 \rightarrow \text{Classes}$) with Dropout(0.3) \\
\midrule
\textbf{\textsc{GSAT Attention}} & \\
\quad \textsc{Attention MLP} & Concat$(x_u, x_v) \rightarrow 3$-Layer MLP ($2H \rightarrow H \rightarrow H \rightarrow 1$) \\
\quad \textsc{Mask Sampling} & Logistic noise + Sigmoid (Train) / Sigmoid (Inference) \\
\quad \textsc{Temperature} ($\tau$) & Linear decay ($1.0 \rightarrow 0.1$) \\
\midrule
\textbf{\textsc{Training}} & \\
\quad \textsc{Loss Function} & CrossEntropy + $\lambda_{\text{info}} \times \text{BCE}(\text{mask}, 1-r)$ \\
\quad \textsc{Default Coeffs} & Prior $r=0.7$, $\lambda_{\text{info}} = 1.0$ \\
\quad \textsc{Hyperparameters} & $H=64$, Adam ($10^{-3}$), Batch 64/128, Epochs 100 \\
\bottomrule
\end{tabular}
\end{table}

\subsection{GINE Architectures (Di-Halo-Benzene)}
\label{app:gine_architectures}

For our edge-attributed dataset, we utilize \textbf{GINE} architectures. These models explicitly utilize the edge features $\mathbf{E}$ defined in Section 2.1.

\textbf{B-cos GINE Formulation} \\
As described in "Handling Edge Attributes" (Section \ref{sec:bcos_gnn}), the B-cos GINE injects edge information $\mathbf{e}_{ij}$ additively before the alignment step. The update rule is:
\begin{equation}
    \mathbf{h}_i^{(k)} = \Psi^{(k)}_{\text{Bcos}} \left( (1 + \epsilon) \mathbf{h}_i^{(k-1)} + \sum_{j \in \mathcal{N}(i)} (\mathbf{h}_j^{(k-1)} + \mathbf{e}_{ji}) \right).
\end{equation}
By treating the edge features as additive modifiers to the neighbor features within the sum aggregation, the entire operation remains within the scope of the B-cos transformation $\Psi^{(k)}_{\text{Bcos}}$, preserving the global dynamic linearity required for explanations.

\textbf{Vanilla GINE Formulation} \\
The Vanilla GINE baseline follows the standard formulation \cite{hu2020strategies}, which applies a ReLU non-linearity during the message construction step $\Phi^{(k)}$:
\begin{equation}
    \mathbf{h}_i^{(k)} = \text{MLP}^{(k)}_{\text{ReLU}} \left( (1 + \epsilon) \mathbf{h}_i^{(k-1)} + \sum_{j \in \mathcal{N}(i)} \text{ReLU}(\mathbf{h}_j^{(k-1)} + \mathbf{e}_{ji}) \right).
\end{equation}
The presence of ReLUs in both the message aggregation and the update MLP $\Psi$ creates a highly non-linear decision boundary.

\textbf{Network Architecture} \\
Consistent with the complexity of these tasks, both GINE variants are deeper ($L=4$) and wider ($d=128$) than the GIN models. Both models utilize the "Readout-Then-Agg" strategy to ensure $\mathbf{\hat y} = \sum \Theta(\mathbf{x}_i^{(L)})$. This setup facilitates a direct comparison of explanation fidelity, as the Vanilla model's logits are also derived from explicit node-level contributions (albeit non-linear ones).

\begin{table}[h]
\centering
\caption{Hyperparameter comparison for GINE architectures (Di-Halo-Benzene).}
\label{tab:gine_comparison}
\begin{small}
\begin{sc}
\begin{tabular}{lcc}
\toprule
\textbf{Parameter} & \textbf{B-cos GINE} & \textbf{Vanilla GINE} \\
\midrule
\textbf{Backbone} & & \\
Conv Layer Type & B-cos GINE & GINEConv \\
Hidden Channels & 128 & 128 \\
Num Layers ($L$) & 4 & 4 \\
$\epsilon$ (eps) & 0.0 (fixed) & Trainable \\
\midrule
\textbf{Readout} & & \\
Strategy & \textbf{Readout-Then-Agg} & \textbf{Readout-Then-Agg} \\
Pooling & Sum & Sum \\
Classifier $\Theta$ & 2-Layer B-cos MLP & 2-Layer ReLU MLP \\
\midrule
\textbf{Training} & & \\
Dropout & 0.5 & 0.5 \\
Optimizer & Adam & Adam \\
$B$ Parameter & 2.0 & N/A \\
\bottomrule
\end{tabular}
\end{sc}
\end{small}
\end{table}

\begin{table}[ht]
\centering
\caption{Architecture and hyperparameters: GSAT + GINE (Compact summary)}
\label{tab:gsat_gine_summary}
\renewcommand{\arraystretch}{1.2}
\begin{tabular}{@{}ll@{}}
\toprule
\textbf{\textsc{Parameter}} & \textbf{\textsc{Specification}} \\
\midrule
\textbf{\textsc{Backbone}} & \\
\quad \textsc{Type} & Vanilla / Maskable GINE (uses node + edge attributes) \\
\quad \textsc{Conv Stack} & 4 Layers (\texttt{aggr='add'}), optional trainable $\epsilon$ \\
\quad \textsc{Hidden Channels} & 64 \\
\quad \textsc{Edge Mask Usage} & Multiplied into message: \texttt{msg * edge\_weight.view(-1,1)} \\
\midrule
\textbf{\textsc{Readout}} & \\
\quad \textsc{Pooling} & SumAggregation \\
\quad \textsc{Classifier $\Theta$} & 2-Layer MLP ($H \rightarrow H/2 \rightarrow \text{Classes}$) with Dropout(0.3) \\
\midrule
\textbf{\textsc{GSAT Attention}} & \\
\quad \textsc{Attention MLP} & Concat($x_u, x_v$, edge\_attr) $\rightarrow$ 3-Layer MLP ($\dots \rightarrow H \rightarrow H \rightarrow 1$) \\
\quad \textsc{Mask Sampling} & Gumbel/Concrete noise (Train) / Sigmoid (Inference) \\
\quad \textsc{Temperature ($\tau$)} & Linear decay ($1.0 \rightarrow 0.1$) \\
\midrule
\textbf{\textsc{Training}} & \\
\quad \textsc{Loss Function} & CE/BCE + $\lambda_{\text{info}} \times \text{BCE}(\text{mask}, 1 - r)$ \\
\quad \textsc{Default Coeffs} & Prior $r=0.7$, $\lambda_{\text{info}} = 1.0$ \\
\quad \textsc{Hyperparameters} & $H \in \{64, 128\}$, Adam/AdamW ($10^{-3}$), Batch 64, Epochs 80/100 \\

\bottomrule
\end{tabular}
\end{table}

\section{Reproducibility Statement}
\label{app:reproducibility}
 
We do not release code with this submission, but provide sufficient information 
needed to reproduce our results:
 
\begin{itemize}
    \item Architecture and hyperparameter specifications for B-cos GIN/GINE 
    and conventional GIN/GINE baselines (App.~\ref{app:gin_architectures}, 
    \ref{app:gine_architectures}, Tabs.~\ref{tab:gin_comparison}, 
    \ref{tab:gine_comparison}).
    \item GSAT baseline specifications, adapted from the authors' reference 
    implementation~\cite{miao2022interpretable} 
    (Tabs.~\ref{tab:gsat_gin_compact}, \ref{tab:gsat_gine_summary}).
    \item Standard benchmarks (PATTERN, NCI1, 
    MolHIV, MNIST-75sp, BA-2Motif) follow their original public 
    distributions.
    \item The Di-Halo-Benzene dataset is provided as supplementary material upon submission.
    \item Optimization details (Adam, learning-rate schedule, early-stopping 
    criterion) in App.~\ref{app:architectures}.
\end{itemize}
 
The B-cos transform implementation follows~\cite{bohle2022b}, whose public code we use unmodified for the per-layer transformation. 
Our contribution is the architectural integration with GIN/GINE, fully specified above. 
We commit to releasing a clean public implementation upon acceptance.

\end{document}